\documentclass{article}
\PassOptionsToPackage{numbers, compress}{natbib}

\usepackage{subfigure}
\usepackage[final]{format/neurips_2022}
\usepackage[utf8]{inputenc} 
\usepackage[T1]{fontenc}    
\usepackage{hyperref}       
\usepackage{url}            
\usepackage{booktabs}       
\usepackage{amsfonts}       
\usepackage{nicefrac}       
\usepackage{microtype}      
\usepackage{xcolor}         

\usepackage[colorinlistoftodos]{todonotes}
\usepackage[shortlabels]{enumitem}
\usepackage[bbgreekl]{mathbbol}

\DeclareSymbolFontAlphabet{\mathbb}{AMSb}
\DeclareSymbolFontAlphabet{\mathbbl}{bbold}

\usepackage{amsmath,amsfonts, amssymb}
\usepackage{mathtools}
\usepackage{amsthm} 
\usepackage{latexsym}
\usepackage{relsize}

\usepackage{hyperref}       
\usepackage{url}            
\usepackage{algorithm}
\usepackage{algorithmic}
\usepackage{braket}

\usepackage{booktabs}       
\usepackage{amsfonts}       
\usepackage[shortlabels]{enumitem}
\usepackage[bbgreekl]{mathbbol}

\DeclareSymbolFontAlphabet{\mathbb}{AMSb}
\DeclareSymbolFontAlphabet{\mathbbl}{bbold}

\usepackage[capitalise]{cleveref}
\usepackage{dsfont}

\usepackage{tikz}
\usetikzlibrary{positioning}
\usepackage{makecell}
\usetikzlibrary{fit}
\usetikzlibrary{shapes.geometric}

\def\W{{\mathcal W}}

\def\E{{\mathbb E}}

\def\W{{\mathcal W}}

\newcommand{\K}{\ensuremath{\mathcal K}}

\newcommand{\ignore}[1]{}

\theoremstyle{plain}
\newtheorem{theorem}{Theorem}
\newtheorem{lemma}[theorem]{Lemma}
\newtheorem{corollary}[theorem]{Corollary}

\newtheorem{claim}[theorem]{Claim}

\newtheorem{assumption}{Assumption}

\newtheorem*{theorem*}{Theorem}
\newtheorem*{lemma*}{Lemma}
\newtheorem*{corollary*}{Corollary}
\newtheorem*{proposition*}{Proposition}
\newtheorem*{claim*}{Claim}
\newtheorem*{fact*}{Fact}
\newtheorem*{observation*}{Observation}
\newtheorem*{assumption*}{Assumption}

\theoremstyle{definition}
\newtheorem{definition}[theorem]{Definition}
\newtheorem{remark}[theorem]{Remark}

\newtheorem*{definition*}{Definition}
\newtheorem*{remark*}{Remark}
\newtheorem*{example*}{Example}

\theoremstyle{plain}
\newtheorem*{theoremaux}{\theoremauxref}
\gdef\theoremauxref{1}

\DeclareMathAlphabet{\mathbfsf}{\encodingdefault}{\sfdefault}{bx}{n}

\DeclareMathOperator*{\argmin}{arg\,min}
\DeclareMathOperator*{\argmax}{arg\,max}

\renewcommand{\O}{\mathcal{O}}

\newcommand{\poly}{\mathrm{poly}}

\newcommand{\reals}{\mathbb{R}}

\renewcommand{\leq}{~\le~}
\renewcommand{\geq}{~\ge~}

\let\oldtfrac\tfrac
\renewcommand{\tfrac}[2]{\smash{\oldtfrac{#1}{#2}}}

\let\nablaold\nabla
\renewcommand{\nabla}{\nablaold\mkern-2.5mu}

\title{A Boosting Approach to Reinforcement Learning}

\author{%
    Nataly Brukhim \\
    Princeton University\\
    \texttt{nbrukhim@cs.princeton.edu}\\
    \And Elad Hazan\\
    Princeton University\\
    Google AI Princeton\\
    \texttt{ehazan@cs.princeton.edu}\\
    \And Karan Singh\\
    Carnegie Mellon University\\
    \texttt{karansingh@cmu.edu}
}

\begin{document}

\maketitle

    \begin{abstract}

Reducing reinforcement learning to supervised learning is a well-studied and effective approach that leverages the benefits of compact function approximation to deal with large-scale Markov decision processes. Independently, the boosting methodology (e.g. AdaBoost) has proven to be indispensable in designing efficient and accurate classification algorithms by combining inaccurate {\em rules-of-thumb}.

In this paper, we take a further step: we reduce reinforcement learning to a sequence of weak learning problems. Since weak learners perform only marginally better than random guesses, such subroutines constitute a weaker assumption than the availability of an accurate supervised learning oracle. We prove that the sample complexity and running time bounds of the proposed method do not explicitly depend on the number of states.

While existing results on boosting operate on convex losses, the value function over policies is non-convex. We show how to use a non-convex variant of the Frank-Wolfe method for boosting, that additionally improves upon the known sample complexity and running time even for reductions to supervised learning.
\end{abstract}

    \section{Introduction}
In reinforcement learning, Markov decision processes (MDP) model the mechanism of learning from rewards, as opposed to examples. Although the case of tabular MDPs is well understood, the main challenge in applying RL in the real-world is the size of the state space in practical domains.

This challenge of finding efficient and provable algorithms for MDPs with large state space is the focus of our study. Various techniques have been suggested and applied to cope with very large MDPs. One class of approaches attempts to approximate either the value or the transition function of the underlying MDP by using a parametric function class. Such approaches invariably make strong {\em realizability assumptions} to produce global optimality guarantees. Another class of approaches, {\em so-called} direct methods, produces a near-optimal policy that maximizes the expected return from a given policy class. To deal with the challenge of large (possibly innumerable) policy classes, a popular strategy \cite{kakade2002approximately} is to the frame policy search as a sequence of supervised learning problems. Such approaches yield global optimality guarantees under state coverage assumptions without reliance on realizability, and have inspired practical adaptations for sampling-based policy search.

In this paper, we study another methodology to derive provable algorithms for reinforcement learning: ensemble methods for aggregating weak or approximate algorithms into substantially more accurate solutions. Our proposal extends the methodology of boosting, typically used to solve supervised learning instances \citep{schapire2012boosting}, to reinforcement learning. A typical boosting algorithm (e.g. AdaBoost) iteratively constructs a near-optimal classifier by combining computationally cheap, yet inaccurate {\em rules-of-thumb}. Unlike RL reductions to supervised learning which assume the existence of an efficient and accurate classification or regression procedure, the proposed algorithms builds on learning algorithms that perform only ever-so-slightly better than a random guess, and which thus may be produced cheaply both in computational and statistical terms.  

Concretely, we assume access to a weak learner: an efficient sample-based procedure that is capable of generating an approximate solution to any weighted multi-class objective over a fixed policy class. We describe an algorithm that iteratively calls this procedure on carefully constructed new objectives, and aggregates the solution into a single policy. We prove that after sufficiently many iterations, our resulting policy has competitive global gurantees on performacnce. Interestingly, unlike boosting algorithms for regression and classification, our resulting aggregation of weak learners is non-linear. 

\subsection{Challenges and techniques}

Reinforcement learning is quite different from supervised learning and several difficulties have to be circumvented for boosting to work. Among the challenges that the reinforcement learning setting presents, consider the following,
\begin{enumerate}[(a)]
    \item  The value function is not a convex or concave function of the policy. This is true even in the tabular case, and even more so if we use a parameterized policy class.

    \item The transition matrix is unknown, or prohibitively large to manipulate for large state spaces. This means that even evaluation of a policy cannot be exact, and can only be computed approximately.

    \item
    It is unrealistic to expect a weak learner that attains near-optimal value for a given linear objective over the policy class. At most one can hope for a multiplicative and/or additive approximation of the overall value.

\end{enumerate}

\begin{table*}
    \begin{tabular}{@{}p{\textwidth}@{}}
        \centering
        \bgroup
        \def\arraystretch{2}
        \begin{tabular}{c|c|c}
            &
            Supervised weak learner &
            Online weak learner 
            \\
            \hline
            Episodic model           & $1/\alpha^4\varepsilon^5$ & $1/\alpha^2\varepsilon^3$ 
            \\
            \hline
            Rollouts w. $\nu$-resets & $1/\alpha^4\varepsilon^6$               & $1/\alpha^2\varepsilon^4$  
            \\
            \hline
        \end{tabular}
        \egroup
    \end{tabular}
    \caption{Sample complexity of the proposed algorithms for different $\alpha$-weak learning models (supervised \& online) and modes of accessing the MDP (rollouts \& rollouts with reset distribution $\nu$), in terms of  $\epsilon$ and $\alpha$,
    suppressing other terms. 
    This work is the first to introduce a reduction of RL to \textit{weak} supervised learning.
    See Theorem~\ref{thm:MAIN1} for details. }
    \label{tab:table1}
\end{table*}
\begin{table*}
    \begin{tabular}{@{}p{\textwidth}@{}}
        \centering
        \bgroup
        \def\arraystretch{1.6}
        \begin{tabular}{c|c|c}
            &
          \multicolumn{2}{c}{Supervised strong learner}
            \\ \hline
            & This work (Corollary~\ref{thm:MAIN3}) & CPI \cite{kakade2002approximately}\\
            \hline
            Episodic model         &    $1/\varepsilon^3$ & $1/\varepsilon^4$
            \\
            \hline
            Rollouts w. $\nu$-resets   & $1/\varepsilon^4$ & $1/\varepsilon^4$
            \\
             \hline
        \end{tabular}
        \egroup
    \end{tabular}
    \caption{Compared to previous work \cite{kakade2002approximately}, the table shows sample complexity of the proposed algorithm for a strong ($\alpha=1$) supervised learning model and different modes of accessing the MDP.} 
    \label{tab:table2}
\end{table*}

Our approach overcomes these challenges by applied several new as well as recently developed techniques. To overcome the nonconvexity of the value function, we use a novel variant of the Frank-Wolfe optimization algorithm that simultaneously delivers on two guarantees. First, it finds a first order stationary point with near-optimal rate. Secondly, if the objective happens to admit a certain gradient domination property, an important generalization of convexity, it also guarantees near optimal value.
The application of the nonconvex Frank-Wolfe method is justified due to previous recent investigation of the policy gradient algorithm \citep{agarwal2019theory,agarwal2020pc}, which identified conditions under which the value function is gradient dominated.

The second information-theoretic challenge of the unknown transition function is overcome by careful algorithmic design: our boosting algorithm requires only samples of the transitions and rewards, obtained by rollouts on the MDP.

The third challenge is perhaps the most difficult to overcome. Thus far, the use of the Frank-Wolfe method in reinforcement learning did not include a multiplicative approximation, which is critical for our application. We adapt the techniques used for boosting in online convex optimization  \citep{hazan2021boosting} with a multiplicative weak learner to our setting, by non-linearly aggregating (using a 2-layer  network) the weak learners. This aspect is perhaps of general interest to boosting algorithm design, which is mostly based on linear aggregation.

\subsection{Our contributions}

Our main contribution is a novel efficient boosting algorithm for reinforcement learning.  Our techniques apply in various settings and the sample complexity bounds of all of our results are summarized in Tables \ref{tab:table1} and \ref{tab:table2}.

The input to this algorithm is a weak learning method capable of approximately solving a weighted multi-class problem instance over a certain policy class.
The output of the algorithm is a policy which does not belong to the original class considered, hence being an instance of {\em improper} learning.
It is rather a non-linear aggregation of policies from the original class, according to a two-layer neural network.
This is a result of the two-tier structure of our algorithm: an outer loop of non-convex Frank-Wolfe method, and an inner loop of online convex optimization based boosting. The final policy comes with provable global optimality guarantees.

Beyond novelty of techniques, an important contribution (Table \ref{tab:table1}) of our work is to highlight the quantitative difference in guarantees that depend on the mode of accessing the MDP (episodic rollouts vs. access to an exploratory reset distrbution) and the nature of the weak learners (online vs statistical), thus indicating that some algorithmic choices may be preferable compared to others in terms of speed of convergence and sample complexity.

As with existing reductions to supervised learning \cite{kakade2002approximately}, these global convergence guarantees happen under appropriate state coverage assumptions either via access to a reset distribution that has some overlap with the state distribution of the optimal policy, or by constraining the policy class to policies that explore sufficiently. Yet another contribution of our work is to show an improved sample complexity result in the latter setting, {\em even when considering reductions to supervised learning instances}. This improvement in convergence in well-studied settings is documented in Table \ref{tab:table2}. 







\subsection{Related work}

Reinforcement learning approaches for dealing with large-scale MDPs rely on function approximation \cite{Sutton1999}. Such function approximation may be performed on the underlying conditional probability of transition (e.g. \cite{sun2019model, jin2020provably}) or the value function (e.g. \cite{weisz2021query, wang2021exponential}). The provable guarantees in such methods come at the cost of strong realizability assumptions. In contrast, the so-called direct approaches attempt policy search over an appropriate policy class \citep{agarwal2019theory,agarwal2020pc}, and rely on making making incremental updates, such as variants of Conservative Policy Iteration (CPI) \citep{kakade2002approximately, scherrer2014local,agarwal2022variance}, and Policy Search by Dynamic Programming (PSDP)\citep{bagnell2003policy}. These provide convergence guarantees under appropriate state coverage assumptions comparable to ones made in this work.

Our boosting approach for provable RL builds on the vast literature of boosting for supervised learning \citep{schapire2012boosting}, and recently online learning \citep{leistner2009robustness, chen2012online, chen2014boosting, beygelzimer2015optimal, jung2017online, jung2018online}.
One of the crucial techniques important for our application is the extension of boosting to the online convex optimization setting, with bandit information \citep{brukhim2021online}, and critically with a multiplicative weak learner \citep{hazan2021boosting}. This latter technique implies a non-linear aggregation of the weak learners. Non-linear boosting was only recently investigated in the context of classification \citep{alon2020boosting}, where it was shown to potentially enable significantly more efficient boosting. Another work on boosting in the context of control of dynamical systems \citep{agarwal2020boosting}. However, this work critically requires knowledge of the underlying dynamics (transitions) and makes convexity assumptions, which we do not, and cannot cope with a multiplicative approximate weak learner.

The Frank-Wolfe algorithm is extensively used in machine learning, see e.g. \citep{jaggi2013revisiting}, references therein, and recent progress in stochastic Frank-Wolfe methods \citep{hassani2017gradient,mokhtari2018stochastic, chen2018projection, xie2019stochastic}.
Recent literature has applied a variant of this algorithm to reinforcement learning in the context of state space exploration \citep{hazan2019provably}.

    \section{Preliminaries}

\paragraph{Optimization.}

We say that a differentiable function $f: \K \mapsto \reals$ over some domain $\K \subset \mathbb{R}^d$ is $L$-smooth with respect to some norm $\|\cdot\|_*$ if for every $x,y \in \K$ we have
$ \left|f(y) - f(x) - \nabla f(x)^\top (y-x)\right| \leq \frac{L}{2} \|x-y\|_*^2 . $
We define the projection $\Gamma:\reals^{|A|} \to \Delta_A$, with respect to a set $A$, where $\Delta_A$  denotes the probability simplex over $A$.
For any $x \in \reals^{|A|}$,
$ \Gamma[x] = \argmin_{y \in \Delta_A}  \| x - y \| . $
An important generalization of the property of convexity we use henceforth is that of gradient domination.
\begin{definition}[Gradient Domination]
    \label{def:grad_dom}
    A function $f:\K\to\reals$ is said to be $(\kappa, \tau, \K_1,\K_2)$-locally gradient dominated (around $\K_1$ by $\K_2$) if for all $x\in \K_1$, it holds that
    $$ \max_{y\in\K}f(y) - f(x) \leq \kappa \cdot \max_{ y\in\K_2  } \left\{ \nabla f(x)^\top (y - x) \right\} +\tau . $$
\end{definition}

\paragraph{Markov decision process.}

An infinite-horizon discounted Markov Decision Process (MDP) $\mathcal{M} = (S, A, P, r, \gamma, d_0)$ is specified by: a state space $S$, an action space $A$, a transition model $P$ where $P(s'|s, a)$ denotes the probability of immediately transitioning to state $s'$ upon taking action $a$ at state $s$, a reward function $r: S \times A \rightarrow [0, 1]$ where $r(s, a)$ is the immediate reward associated with taking action $a$ at state $s$, a discount factor $\gamma \in [0,1)$; a starting state distribution $d_0$ over $S$.
For any infinite-length state-action sequence (hereafter, called a trajectory), we assign the following value
$ V(\varsigma = (s_0,a_0,s_1,a_1,\dots)) = \sum_{t=0}^\infty \gamma^t r(s_t,a_t) . $ The agent interacts with the MDP through the choice of stochastic policy $\pi:S\to \Delta_A$ it executes.
The execution of such a policy induces a distribution over trajectories $\varsigma=(s_0,a_0,\dots)$ as
$    P(\varsigma|\pi) = d_0(s_0) \prod_{t=0}^\infty (P(s_{t+1}|s_t,a_t)\pi(a_t|s_t)).\label{eq:traj}
$ Using this description we can associate a state $V^\pi(s)$ and state-action $Q^\pi(s,a)$ value function with any policy $\pi$.
For an arbitrary distribution $d$ over $S$, define:
$$ Q^\pi(s, a) = \E \Bigg[ \sum_{t=0}^{\infty} \gamma^t r(s_t, a_t)
\Big|\  \pi, s_0=s, a_0=a \Bigg], $$
$$ V^\pi(s) = \E_{a\sim \pi(\cdot|s)}\left[ Q^\pi(s,a) | \pi, s \right], \ V^\pi_{d} = \E_{s_0\sim d} \left[V^\pi(s) | \pi\right]. $$
Here the expectation is with respect to the randomness of the trajectory induced by $\pi$ in $\mathcal{M}$.
When convenient, we shall use $V^\pi$ to denote $V^\pi_{d_0}$, and $V^*$ to denote $\max_\pi V^\pi$.

Similarly, to any policy $\pi$, one may ascribe a (discounted) state-visitation distribution $d^\pi = d^\pi_{d_0}$.
$$ d^\pi_{d} (s) = (1-\gamma)\sum_{t=0}^\infty \gamma^t \sum_{\varsigma: s_t =s} P(\varsigma|\pi, s_0\sim d) $$

\paragraph{Modes of Accessing the MDP.}\label{subsec:mdp_access}

We henceforth consider two modes of accessing the MDP, that are standard in the reinforcement learning literature, and provide different results for each.

The first natural access model is called the {\bf episodic rollout setting.} This mode of interaction allows us to execute a policy, stop and restart at any point, and do this multiple times.


Another interaction model we consider is called {\bf rollout with $\nu$-restarts.} This is similar to the episodic setting, but here the agent may draw from the MDP a trajectory seeded with an initial state distribution $\nu\neq d_0$.
This interaction model was considered in prior work on policy optimization \cite{kakade2002approximately,agarwal2019theory}.
The motivation for this model is two-fold: first, $\nu$ can be used to incorporate priors (or domain knowledge) about the state coverage of the optimal policy; second, $\nu$ provides a mechanism to incorporate exploration into policy optimization procedures.



\subsection{Weak learning}
Our boosting algorithms henceforth call upon weak learners to generate weak policies. We formalize the notion of a weak learner next. We consider two types of weak learners, and give different end results based on the different assumptions: weak supervised and weak online learners. In the discussion below, let $\pi_{Rand}$ be a uniformly random policy, i.e. $\forall (s,a)\in S\times A, \pi_{Rand}(a|s)=1/|A|$. The formal definition and results for the online setting are deferred to the appendix. In what follows we define the supervised weak learning model. 



The natural way to define weak learning is an algorithm whose performance is always slight better than that of random policy, one that chooses an action uniformly at random at any given state.
However, in general no learner can outperform a random learner over all label distributions. This motivates the literature on agnostic boosting \citep{kanade2009potential,brukhim2020online,hazan2021boosting} that defines a weak learner as one that can approximate the best policy in a given policy class.

\begin{definition}[Weak Supervised Learner]
    \label{def:wl}
    Let $\alpha \in (0,1]$.
    Consider a class $\mathcal{L}$ of linear loss functions
    $\ell:\reals^A \to \reals$, a family $\mathbb{D}$ of distributions that are supported over $S\times \mathcal{L}$, and policy class $\Pi$. A weak supervised learning algorithm, for every $\varepsilon, \delta > 0$, given $m(\varepsilon,\delta)=\frac{\log |{\W}|}{\varepsilon^2}\log \frac{1}{\delta}$ samples $D_m$ from any distribution $\mathcal{D} \in \mathbb{D}$ outputs a policy $\mathcal{W}(D_m)\in\Pi$ such that with probability $1-\delta$,
    \begin{align*}
        \mathbb{E}_{(s,\ell)\sim \mathcal{D}} \big[\ell(\mathcal{W}(D_m))\big] &\leq
        \alpha
        \min_{\pi^*\in\Pi}\mathbb{E}_{(s,\ell)\sim \mathcal{D}} \big[\ell(\pi^*(s))\big] + (1-\alpha) \ 
        \mathbb{E}_{(s,\ell)\sim \mathcal{D}} \big[\ell(\pi_{Rand}(s))\big] + \varepsilon.
    \end{align*}
\end{definition}

Note that the weak learner outputs a policy in $\Pi$ which is approximately competitive against the class $\Pi$.
As an additional relaxation, instead of requiring that the weak learning guarantee holds for all distributions, in our setup, it will be sufficient that the weak learning assumption holds over \emph{natural} distributions. Specifically, we define a class of \emph{natural} distributions $\mathbb{D}$, such that $\mathcal{D}\in \mathbb{D}$  
    if and only if there exists some $\pi\in\mathbbl{\Pi}$ such that,
    $ \mathcal{D}(s) = \int_\ell \mathcal{D}(s,\ell) d\mu(\ell) = d^\pi(s). $ In particular, while a \emph{natural} distribution may have arbitrary distribution over labels, its marginal distribution over states must be realizable as the state distribution of some policy in $\mathbbl{\Pi}$ over the MDP $\mathcal{M}$.
Therefore, the complexity of weak learning adapts to the complexity of the MDP itself.
As an extreme example, in stochastic contextual bandits where policies do not affect the distribution of states (say $d_0$), it is sufficient that the weak learning condition holds with respect to all couplings of a single distribution $d_0$.

    \section{Algorithm \& Main Results}\label{sec:main}

In this section we describe our RL boosting algorithm. Here we focus on the case where a supervised weak learning is provided. The online weak learners variant of our result is detailed in the appendix.
We next define several definitions and algorithmic subroutines required for our method.

\subsection{Policy aggregation}

For a base class of policies $\Pi$, our algorithm incrementally builds a more expressive policy class by aggregating base policies via both linear combinations and non-linear transformations.
In effect, the algorithm produces a finite-width depth-2 circuit over some subset of the base policy class. That is, our approach can be thought of as an aggregation of base policies, which forms a 2-layer neural network, as depicted in Figure \ref{fig:M1}. The leaves of the tree are the policies $\pi \in \Pi$ the base policy class. These are then linearly aggregated to form the first layer of the tree, denoted $\Tilde{\pi}_1, \Tilde{\pi}_2$ in Figure \ref{fig:M1}. 
 
Next, each linear combination of policies in the overall aggregation undergoes 
a projection operation. The projection may be viewed as a non-linear activation function, such as ReLU, in deep learning terms.
Note that the projection of any function from $S$ to $\reals^{|A|}$ produces a policy, i.e. a mapping from states to distributions over actions. In the analysis of our algorithm we give a particular projection operation $\Gamma[\cdot]$ which allows us to yield the desired guarantees. 

\begin{definition}[Policy Projection]
    Given $\Tilde{\pi}:S\to \reals^{|A|}$, define a projected policy $\pi=\Gamma[\Tilde{\pi}]$ to be a policy such that simultaneously for all $s\in S$, it holds that $\pi(\cdot | s) = \Gamma\left[\Tilde{\pi}(s)\right]. $
\end{definition}

\begin{definition}[Policy Tree]
    \label{def:pol_tree}
   A \textit{Policy Tree} $\mathbbl{\Pi} \subseteq S\to\Delta_A$ with respect to $\Pi \subseteq S\to\Delta_A$ some  base policy class, and $N, T \in \mathbb{N}$,
 is a linear combination of $T$ projected policies $\Gamma[\Tilde{\pi}]$, where each $\Tilde{\pi}$ is a linear combination of $N$ base policies $\pi \in \Pi$.
\end{definition}

This final definition describes the set of possible outputs of the boosting procedure. It is important that the policy that the boosting algorithm outputs can be evaluated efficiently. In the appendix we show it is indeed the case (see Lemma \ref{claim:eff_eval}). Hereafter, we refer to a Policy Tree with respect to $\Pi$, $N$ and $T$, as $\mathbbl{\Pi}$ for $N,T=O(\poly(|A|,(1-\gamma)^{-1},\varepsilon^{-1},\alpha^{-1},\log \delta^{-1}))$ specified later.

\tikzset{
    solid node/.style={circle,draw,inner sep=1.2,fill=black},
    hollow node/.style={circle,draw,inner sep=2.3},
    square/.style={rectangle,draw,inner sep=1.5, fill=black},
}
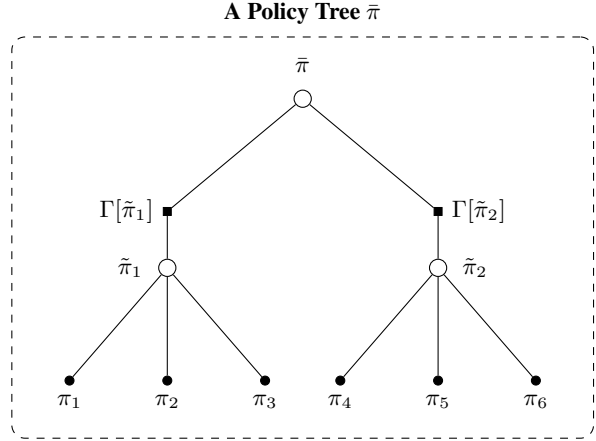
\begin{figure}
    \centering
    \begin{tikzpicture}[font=\footnotesize]
        \tikzset{
            level 1/.style={level distance=15mm,sibling distance=36mm},
            level 2/.style={level distance=7.5mm,sibling distance=10mm},
            level 3/.style={level distance=15mm,sibling distance=13mm},
            level 4/.style={level distance=15mm,sibling distance=20mm},
        }
        \node(t0)[hollow node,label=above:{
        \makecell{\small{$\bar{\pi}$}}}]{}
        child{node(a1)[square,label=left:{$\Gamma[\Tilde{\pi}_1]$}]{}
        child{node(l1)[hollow node,label=left:{$\Tilde{\pi}_1 \ $}]{}
        child{node(b1)[solid node,label=below:{$\pi_1$}]{}edge from parent node[left]{}}
        child{node(b2)[solid node,label=below:{$\pi_2$}]{}edge from parent node[left]{}}
        child{node(b3)[solid node,label=below:{$\pi_3$}]{}edge from parent node[right]{}}
        edge from parent node[left]{}
        }edge from parent node[left]{$ \ \ \ $}
        }
        child{node(a2)[square,label=right:{$\Gamma[\Tilde{\pi}_2]$}]{}
        child{node(l2)[hollow node,label=right:{$\ \Tilde{\pi}_2$}]{}
        child{node[solid node,label=below:{$\pi_4$}]{}edge from parent node[left]{}}
        child{node[solid node,label=below:{$\pi_5$}]{}edge from parent node[left]{}}
        child{node(b6)[solid node,label=below:{$\pi_6$}]{}edge from parent node[right]{}}
        edge from parent node[right]{}
        }      edge from parent node[right]{}};
        \node [fit=(b1)(b6)(l1)(t0),inner sep=0.7cm,dashed,draw,rounded corners=5pt,label=above:{\makecell{\small{\textbf{A Policy Tree}} $\bar{\pi} $}}]{};
    \end{tikzpicture}
    \caption{The figure illustrates a Policy Tree hierarchy (see Definition \ref{def:pol_tree}), output of 
the boosting procedure specified in Algorithm \ref{alg:MAIN1}. Specifically, it is 
obtained by setting $N=3$ on the inner loop of Internal Boost (Algorithm \ref{alg:int_boo}), and $T=2$ on the main booster (Algorithm \ref{alg:MAIN1}). Overall we get all base policies $\pi_1,...,\pi_6 \in \Pi$  on the lower level, to form the Policy Tree $\bar{\pi} \in \mathbbl{\Pi}$.
    } \label{fig:M1}
\end{figure}

\subsection{Main results}

\begin{algorithm}[t]
    \caption{RL Boosting}\label{alg:MAIN1}
    \begin{algorithmic}[1]
        \STATE \textbf{Input}: number of iterations $T$, initial state distribution $\mu$, and $P, N,M$ parameters for Internal Boost.
        \STATE Initialize a policy $\pi_0\in \Pi$ arbitrarily. 
        \FOR{$t=1$ {\bfseries to} $T$}
        \STATE Run Internal Boost (Algorithm \ref{alg:int_boo}) with distribution $\mu$ and policy $\pi_t$ to obtain $\pi_t'$.
        \STATE Update $\pi_t = (1-\eta_{1,t})\pi_{t-1} + \eta_{1,t} \pi'_t$.
        \ENDFOR
    \STATE Run each policy $\pi_t$ for $P$ rollouts to compute an empirical estimate $\widehat{V^{\pi_t}}$  of the expected return.
    \RETURN $\bar{\pi} := \pi_{t'}$ where $t'=\argmax_t \widehat{V^{\pi_t}}$.
    \end{algorithmic}
\end{algorithm}

\begin{algorithm}[t]
    \caption{Internal Boost}\label{alg:int_boo}
    \begin{algorithmic}[1]
        \STATE \textbf{Input}: number of iterations $N$, number of episodes $M$, initial policy $\pi$, initial state distribution $\mu$.
        \STATE Set $\Tilde{\pi}_{0}$ to be an arbitrary policy in $\Pi$.
        \FOR{$n=1$ {\bfseries to} $N$}
        \STATE Execute $\pi$ with $\mu$ via Algorithm \ref{alg:q_sampler} for $M$ episodes, to get         $D_{n}=\{(s_i, \widehat{Q_i})_{i=1}^M\}.$
        \STATE Modify $D_{n}$ to produce a new dataset $D'_{n}=\{(s_i, f_i)\}_{i=1}^M$, such that for all $i \in [m]$:
        $$ f_{i} =  \frac{1}{\beta} \left(y_i - \Tilde{\pi}_{n}(\cdot|s_i) \right), \quad         y_i = \argmin_{y\in \reals^{|A|}}\Large\{ -\widehat{Q}_i^\top y + G\min_{z\in \Delta_A}\|z-y\|+\frac{\|\Tilde{\pi}_{n}(\cdot|s_i)-y\|^2}{2\beta} \Large\} $$
        where  $G= \frac{A}{1-\gamma}, \beta = \frac{2\gamma}{(1-\gamma)^3} $ and $f_i, \widehat{Q}_i \in \reals^{|A|}$.

        \STATE Let $\mathcal{A}_{n}$ be the policy chosen by the weak learning oracle when given data set $D'_{t,n}$.
        \STATE Update 
        $$
        \Tilde{\pi}_{n} = (1-\eta_{2,n})\Tilde{\pi}_{n-1} + \frac{\eta_{2,n}}{\alpha} \mathcal{A}_{n}.
        $$
        \ENDFOR
        \RETURN $\Gamma\left[\Tilde{\pi}_{N}\right]$.
     \end{algorithmic}
\end{algorithm}
\begin{algorithm}[t]
    \begin{algorithmic}[1]
        \STATE Sample state $s_0 \sim \mu$, action  $a' \sim \mathcal{U}(A)$ uniformly.
        \STATE Sample $s\sim d^{\pi}$ as follows:
        at every timestep $h$,  with probability $\gamma$, act
        according to $\pi$; else, accept  $s_h$ as the sample and
        proceed to Step 3. 
        \STATE\label{state:next_step} Take action $a'$ at state $s_h$, then
        continue to execute $\pi$, and use a termination probability of
        $1-\gamma$. Upon termination, set
        $R(s_h,a')$ as the \emph{undiscounted} sum
        of rewards from time $h$ onwards.
        \STATE Define the vector $\widehat{Q^\pi_{s_h}}$, such that for all $a \in A$, $\widehat{Q^\pi_{s_h}}(a) = |A|\cdot R(s_h,a') \cdot \mathbb{I}_{a=a'}$.
        \RETURN $(s_h, \widehat{Q^\pi_{s_h}})$.
    \end{algorithmic}
    \caption{Trajectory Sampler: samples a state $s \sim d^{\pi}$, and an unbiased estimate of $Q^\pi_s$}
    \label{alg:q_sampler}
\end{algorithm}

Next, we give the main results of our RL boosting algorithm via weak supervised learning, specified in Algorithm \ref{alg:MAIN1}.


To state the results, we need the following definitions. The first generalizes the policy completeness notion from \citep{scherrer2014local}. It may be seen as the policy-equivalent analogue of inherent bellman error \citep{munos2008finite}. Intuitively, it measures the degree to which a policy in $\Pi$ can best approximate the bellman operator in an average sense with respect to the state distribution induced by a policy from $\mathbbl{\Pi}$.

\begin{definition}[Policy Completeness]
    For any initial state distribution $\mu$, and policy classes $\Pi, \mathbbl{\Pi}$,
    define
$       \mathcal{E}_\mu =
         \max_{\pi\in \mathbbl{\Pi}} \min_{\pi^*\in \Pi} \E_{s\sim d^\pi_\mu} \left[ \max_{a\in A}Q^\pi(s,a) - Q^\pi(s,\cdot )^\top \pi^*(\cdot | s) \right].
$\end{definition}

\begin{definition}[Distribution Mismatch]
    Let $\pi^* = \argmax_\pi V^\pi$, and $\nu$ a fixed initial state distribution (see section \ref{subsec:mdp_access}). Define the following distribution mismatch coefficients:
    $ C_\infty= \max_{\pi\in \mathbbl{\Pi}}\left\|{d^{\pi^*}}/{d^\pi}\right\|_\infty,  D_\infty = \left\|{d^{\pi^*}}/{\nu}\right\|_\infty. $
\end{definition}

The above notion of the distribution mismatch coefficient is often useful to characterize the exploration problem faced by policy optimization algorithms. We now give the main result for the output of our RL boosting algorithm, assuming supervised weak learners.

\begin{theorem}
    \label{thm:MAIN1}
    Algorithm~\ref{alg:MAIN1} samples $T(MN+P)$ episodes of length $\tilde{O}(\frac{1}{1-\gamma})$ with probability $1-\delta$.\\
    In the \underline{episodic model}, with $\mu =d_0$, for $\eta_{1,t} = \min\{1,\frac{2C_\infty}{t}\}$,  $T=O\left(\frac{C^2_\infty}{(1-\gamma)^3 \varepsilon}\right)$, $N=\left(\frac{16|A|C_\infty}{(1-\gamma)^2\alpha\epsilon}\right)^2$, $M=m\left(\frac{(1-\gamma)^2\alpha\varepsilon}{C_\infty|A|},\frac{\delta}{NT}\right)$,   with probability $1-\delta$,
    $ V^* - V^\pi \leq  \frac{C_\infty\mathcal{E}}{1-\gamma} + \varepsilon.$\\
    In the \underline{$\nu$-reset model}, with $\mu=\nu$, for  $\eta_{1,t} = \sqrt{\frac{8\gamma (1-\gamma)^2}{|A|^2 T}}$, $T=\frac{8D^2_\infty}{(1-\gamma)^6 \varepsilon^2}$, $N=\left(\frac{16|A|D_\infty}{(1-\gamma)^3\alpha\epsilon}\right)^2$,  $M=m\left(\frac{(1-\gamma)^3\alpha\varepsilon}{8|A|D_\infty},\frac{\delta}{2NT}\right)$,  with probability $1-\delta$,
    $ V^* - V^\pi \leq  \frac{D_\infty\mathcal{E}_\nu}{(1-\gamma)^2} + \varepsilon.$\\
    \textbf{Sample complexities:} If $m(\varepsilon,\delta)=\frac{\log |\W|}{\varepsilon^2}\log \frac{1}{\delta}$ for some measure of weak learning complexity $|\W|$, the algorithm samples $\tilde{O}\left(\frac{C_\infty^6 |A|^4\log |\W|}{(1-\gamma)^{11} \alpha^4\varepsilon^5}\right)$ episodes in the episodic model, and $\tilde{O}\left(\frac{D_\infty^6|A|^4\log |\W|}{(1-\gamma)^{18} \alpha^4\varepsilon^6}\right)$ in the $\nu$-reset model.
\end{theorem}
Theorem \ref{thm:MAIN1} above pertains to the case where a weak learning algorithm is available. 
However, another main result is given by considering the simpler approach of reduction of RL to a \textit{strong} supervised learning algorithm. In particular, when running our main boosting algorithm, we can replace the call to Internal Boost (in Line 4 of Algorithm \ref{alg:MAIN1}) with a call to a \textit{strong} supervised learning algorithm. By a similar analysis to that of Theorem \ref{thm:MAIN1} we obtain the following corollary. 

\begin{corollary} \label{thm:MAIN3}
Let $m(\varepsilon,\delta)=\frac{\log |\W|}{\varepsilon^2}\log \frac{1}{\delta}$ for some measure of weak learning complexity $|\W|$.
When run with a supervised learning oracle (Definition \ref{def:wl} with $\alpha=1$, i.e. $N=1$) as the Internal boosting, Algorithm~\ref{alg:MAIN1} samples $\tilde{O}\left(\frac{C_\infty^3 \log |\W|}{ \varepsilon^3}\right)$ episodes in the episodic model, and $\tilde{O}\left(\frac{D_\infty^4\log |\W|}{ \varepsilon^4}\right)$ in the $\nu$-reset model, to guarantee $V^* - V^\pi \leq  \frac{C_\infty\mathcal{E}}{1-\gamma} + \varepsilon$ with probability $1-\delta$ in the episodic model and $V^* - V^\pi \leq  \frac{D_\infty\mathcal{E}_\nu}{(1-\gamma)^2} + \varepsilon$ in the $\nu$-reset model.
\end{corollary} 
 
We note that this result is an improvement over previous results in terms of sample complexity requirement of the algorithm. In particular, in \cite{kakade2002approximately}, Theorem 4.4 and Corollary 4.5 achieve the same guarantee using $O(1/\varepsilon^4)$ samples regardless of the MDP access model. Briefly, CPI utilizes $1/\varepsilon^2$ calls to an $\varepsilon$-optimal supervised learning oracle (each call needing $1/\varepsilon^2$ samples) to reach a $\varepsilon$-local optima of the value function. Under requisite state coverage assumptions, this translates to $\varepsilon$-function value suboptimality. Indeed, such mode of analysis via first arguing for convergence to a local optima for the CPI algorithm can be shown to be tight. The improvement in our case for the episodic access model comes from the insight that it is possible to make direct claims on the function value sub-optimality (second part of Theorem \ref{fwthm}), bypassing the need for making a claim on the local optimality, in the gradient-dominated case.

\subsection{Trajectory sampler}

In Algorithm \ref{alg:q_sampler} we describe an episodic sampling procedure, that is used in our sample-based RL boosting algorithms described above. For a fixed initial state distribution $\mu$, and any given policy $\pi$, we apply the following sampling procedure: start at an initial state $s_0\sim \mu$, and continue
to act thereafter in the MDP according to any policy $\pi$, until termination. With this process, it is straightforward
to both sample from the state visitation distribution $s \sim d^{\pi}$, and to obtain unbiased samples of $Q^{\pi}(s,\cdot)$; see
Algorithm~\ref{alg:q_sampler} for the detailed process.
    
\section{Sketch of the analysis}
{\bf Non-convex Frank-Wolfe.} We give an abstract high-level procedural template that the previously introduced RL boosters operate in. This is based on a variant of the Frank-Wolfe optimization technique \cite{frank1956algorithm}, adapted to non-convex and gradient dominated function classes (see Definition \ref{def:grad_dom}). The Frank-Wolfe (FW) method assumes oracle access to a black-box linear optimizer, denoted $\mathcal{O}$, and utilizes it by iteratively making oracle calls with modified objectives, in order to solve the harder task of convex optimization. Analogously, boosting algorithms often assume oracle access to a ”weak” learner, which are utilized by iteratively making oracle calls with modified objective, in order to obtain a ”strong” learner, with boosted performance. In the RL setting, the objective is in fact non-convex, but exhibits gradient domination. By adapting Frank-Wolfe technique to this setting, we will in subsequent section obtain guarantees for the algorithms given in Section \ref{sec:main}.{\bf Oracle:} Denote by $\O$ a black-box oracle to an $(\epsilon_0,\K_2)$-approximate linear optimizer over a convex set $\K \subseteq \mathbb{R}^d$ such that for any given $v\in \reals^d$, we have
$v^\top \O(v) \geq \max_{u \in \K_2} v^\top u - \epsilon_0.$
\begin{algorithm}
    \caption{Non-convex Frank-Wolfe}
    \begin{algorithmic}[1] \label{fw2} 
        \STATE Input: $T > 0$, objective $f$, linear optimizer $\O$, rate $\eta_t$.
        \STATE Choose $x_0 \in \K$ arbitrarily.
        \FOR{$t = 1, \ldots, T$}
        \STATE Call $ z_t =  \O(\nabla_{t-1}) $, where $\nabla_{t-1} = \nabla f(x_{t-1})$.
         Set $x_{t} = (1 - \eta_t) x_{t-1} + \eta_t z_t$.
        \ENDFOR
        \RETURN $\bar{x}:=x_{t'}$ where $t'=\argmin_t \nabla_{t}^\top (z_t - x_t)$. 
    \end{algorithmic}
\end{algorithm}

\begin{theorem}
    \label{fwthm}
    Let $f:\K\to\reals$ be $L$-smooth in some norm $\|\cdot\|_*$, bounded for all $x \in \K$, $|f(x)| \le H$ for some $H>0$, and let the diameter of $\K$ in $\|\cdot\|_*$ be $D$. Then, for a $(\epsilon_0,\K_2)$-linear optimization oracle $\mathcal{O}$, and $\eta_t  = \eta =  \sqrt{\frac{4H}{LD^2 T}}$, the output $\bar{x}$ of Algorithm \ref{fw2} satisfies
    $$  \max_{u\in \K_2}\nabla f(\bar{x})^\top ( u - \bar{x})  \leq   \sqrt{\frac{2HLD^2}{T}}+\epsilon_0 ;  \max_{x^*\in \K}f(x^*) - f(\bar{x})   \leq   \frac{2 \kappa^2\max\{L D^2,H\}}{T} + \tau+\kappa\epsilon_0 $$
    Furthermore, if $f$ is $(\kappa,\tau,\K_1,\K_2)$-locally gradient-dominated and $x_0,\dots x_T\in \K_1$, then the output $\bar{x}$ of Algorithm \ref{fw2} where
    $\eta_t = \min\{1,\frac{2\kappa}{t}\}$
     satisfies the bound on the right.
\end{theorem}

 We sketch the high-level ideas of the proof of our main result, stated in Theorem~\ref{thm:MAIN1}, and refer the reader to the appendix for the formal proof.  We will establish an equivalence between RL Boosting (Algorithm \ref{alg:MAIN1}) and the variant of the Frank-Wolfe algorithm (Algorithm~\ref{fw2}). This abstraction allows us to obtain the novel convergence guarantees given in Theorem \ref{thm:MAIN1}. Throughout the analysis, we use the notation $\nabla_\pi V^\pi$ to denote the gradient of the value function with respect to the $|S|\times|A|$-sized representation of the policy $\pi$, namely the functional gradient of $V^\pi$.


{\bf Internal-boosting weak learners.}
The Frank-Wolfe algorithm utilizes an inner gradient optimization oracle as a subroutine. To implement this oracle using approximate optimizers, we utilize yet another variant of the FW method as ``internal-boosting'' for the weak learners, by employing an adapted analysis of \cite{hazan2021boosting} that is stated in Claim \ref{claim:oldstat} below. 
Let $\mathcal{D}_t$ be the distribution induced by the trajectory sampler in round $t$.

\begin{claim}
    \label{claim:oldstat}
    Let $\beta = \sqrt{{1}/{\alpha N}}$, $\eta_{2,n}=\min\{{2}/{n},1\}$. $\pi'_t$ produced by Algorithm~\ref{alg:MAIN1} satisfies
$     \max_{\pi\in\Pi}\mathbb{E}_{(s,Q)\sim \mathcal{D}_t} \left[ Q^\top \pi(s) \right]  
    - \mathbb{E}_{(s,Q)\sim \mathcal{D}_t} \left[ Q^\top \pi'_t(s) \right] \leq ({2|A|}/(1-\gamma)\alpha) \left(\varepsilon+ {2}/{\sqrt{N}}  \right).
$
\end{claim}

{\bf From weak learning to linear optimization,} Next, we give an important observation which allows us to re-state the guarantee in the previous subsection in terms of linear optimization over functional gradients. The key observation here is that the expensive  optimizing procedure for $(\nabla_{\pi}V^\pi)^\top \pi'$, which in particular requires iterating over all states in $S$, can be instead replaced with sampling from an appropriate distribution $\mathcal{D}$ (via Algorithm \ref{alg:q_sampler}). These sample pairs $(s, \widehat{Q^\pi}(s,\cdot))$ could then be fed to our weak learning algorithm, which guarantees generalization. 
\begin{lemma}
    \label{lemma:q_sampler}
    Applying Algorithm \ref{alg:q_sampler} for any given policy $\pi$ yields an unbiased estimate of the gradient, such that for any $\pi'$,
        $
        (\nabla_{\pi}V_\mu^\pi)^\top \pi' = \E_{(s, \widehat{Q^\pi}(s,\cdot)) \sim \mathcal{D}}\Big[
        \widehat{Q^\pi}(s,\cdot)^\top  \pi'(\cdot|s)
        \Big]/{(1-\gamma)}$,
    where $\pi'(\cdot|s) \in \Delta_A$, $\mathcal{D}$ is the distribution induced on the outputs of Algorithm \ref{alg:q_sampler},  for the policy $\pi$ and initial  distribution $\mu$.
\end{lemma}
\section{Experiments}
The primary contribution of the present work is theoretical. Nevertheless, we empirically test our proposal with the experiment designed to elicit qualitative properties of the proposed algorithm, instead of aiming to achieve the state-of-the-art. To validate our results, we check if the proposed algorithm is indeed capable of boosting the accuracy of concrete instantiations of weak learners. We use depth-3 decision trees, with the implementation adapted from Scikit-Learn \cite{pedregosa2011scikit}, as our base weak learner. This choice of weak learner is particularly suitable for boosting, because it is an impoverished policy class in a representational sense and hence it is reasonable to expect that it may do only slightly better than random guessing with respect to the classification loss.
We consider the performance of the boosting algorithm (Algorithm 1) across multiple rounds of boosting or number of weak learners to that of supervised-learning-based policy iteration; the computational burden of the algorithm scales linearly with the latter. Throughout all the experiments, we used $\eta=0.9$. To speed up computation, the plots below were generated by retaining the 3 most recent policies of every iteration in the policy mixture. We evaluated these on the CartPole and the LunarLander environments. The results demonstrate the proposed RL boosting algorithm succeeds in maximizing rewards while using few weak learners (equivalently, within a few rounds of boosting).

\begin{figure}[!h] \label{fig:plots_ci}
     \centering
\includegraphics[scale=0.4]{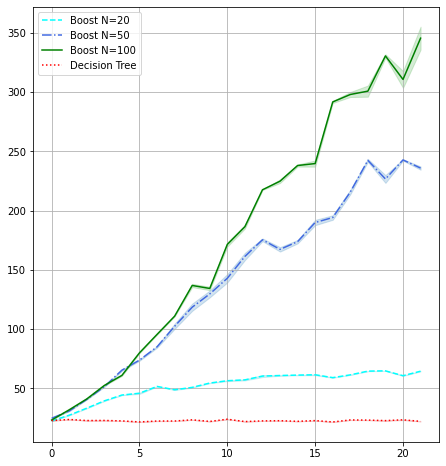}
 \includegraphics[scale=0.4]{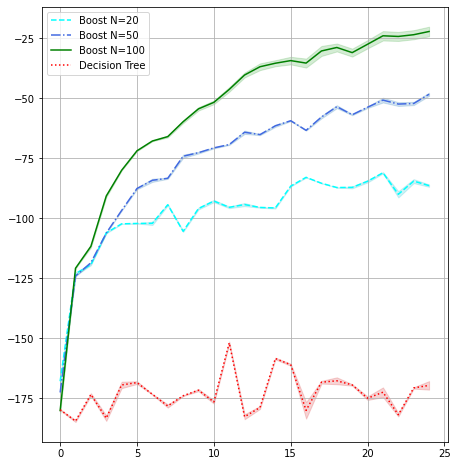}
    \caption{Reward trajectory for the CartPole (left) and the LunarLander (right) environments
    of the proposed boosting algorithm for $N=20,50,100$ number of base weak learners is compared to supervised-learning-based policy iteration (decision tree) above. The x-axis corresponds to $T$ number of iterations, and for each $t \in [T]$, reward is computed over $100$ episodes of interactions. The confidence interval is plotted over 3 such runs.}
\end{figure}

\section{Conclusions}
Building on recent advances in boosting for online convex optimization and bandits, we have described a boosting algorithm for reinforcement learning over large state spaces with provable guarantees. We see this as a first attempt at bringing a tried-and-tested methodology from supervised learning to RL. 

Many avenues of future research arise from the present work. Some of these are listed below.
\begin{itemize}
	\item Can the methodology of boosting be extended to other sequential decision making and interactive learning problems? Such a possibility for off-policy learning was recently explored following the release of a previous version of the present work in \cite{london2022boosted}, albeit under the requirement of a stronger weak learner.
	\item Our results underscore the difference between online and statistical weak learners in terms of the final sample complexity each achieves. Is there a fundamental difference between these weak learning models in the context of reinforcement learning? Can the difference between these sample complexities be minimized using a different analysis or choice of boosting algorithm?
	\item The two-level boosting scheme utilizes Conservative Policy Iteration \cite{kakade2002approximately} for the outer loop, for which we also present an improved analysis under stronger state coverage assumptions. Recently, the work of \cite{agarwal2022variance} introduced a modification of CPI that works with asymptotically fewer samples, when reducing RL to supervised learning problems. Can such techniques be leveraged to reduce the sample requirement for reductions to weak learning?
	\item The weak learner considered in the present work optimizes a linear function over policy space. Another natural weak learner would be an RL agent with multiplicatively approximate optimality guarantee, and it would be interesting to extend our methodology to this notion of weak learnability.
	\item Another important aspect that is not discussed in this work is that of state-space exploration. Potentially boosting can be combined with state-space exploration techniques to give stronger guarantees independent or only weakly dependent of distribution mismatch $C_\infty,D_\infty$ factors.
	\item A feature of our proposal is that it produces twice-aggregated nonlinear combinations of weak learners. Are simpler aggregations with provable guarantees possible?
\end{itemize}


%


\bibliography{paper_icml/bib}

\begin{thebibliography}{8}
\providecommand{\natexlab}[1]{#1}
\providecommand{\url}[1]{\texttt{#1}}
\expandafter\ifx\csname urlstyle\endcsname\relax
  \providecommand{\doi}[1]{doi: #1}\else
  \providecommand{\doi}{doi: \begingroup \urlstyle{rm}\Url}\fi

\bibitem[Author(2021)]{anonymous}
Author, N.~N.
\newblock Suppressed for anonymity, 2021.

\bibitem[Duda et~al.(2000)Duda, Hart, and Stork]{DudaHart2nd}
Duda, R.~O., Hart, P.~E., and Stork, D.~G.
\newblock \emph{Pattern Classification}.
\newblock John Wiley and Sons, 2nd edition, 2000.

\bibitem[Kearns(1989)]{kearns89}
Kearns, M.~J.
\newblock \emph{Computational Complexity of Machine Learning}.
\newblock PhD thesis, Department of Computer Science, Harvard University, 1989.

\bibitem[Langley(2000)]{langley00}
Langley, P.
\newblock Crafting papers on machine learning.
\newblock In Langley, P. (ed.), \emph{Proceedings of the 17th International
  Conference on Machine Learning (ICML 2000)}, pp.\  1207--1216, Stanford, CA,
  2000. Morgan Kaufmann.

\bibitem[Michalski et~al.(1983)Michalski, Carbonell, and
  Mitchell]{MachineLearningI}
Michalski, R.~S., Carbonell, J.~G., and Mitchell, T.~M. (eds.).
\newblock \emph{Machine Learning: An Artificial Intelligence Approach, Vol. I}.
\newblock Tioga, Palo Alto, CA, 1983.

\bibitem[Mitchell(1980)]{mitchell80}
Mitchell, T.~M.
\newblock The need for biases in learning generalizations.
\newblock Technical report, Computer Science Department, Rutgers University,
  New Brunswick, MA, 1980.

\bibitem[Newell \& Rosenbloom(1981)Newell and Rosenbloom]{Newell81}
Newell, A. and Rosenbloom, P.~S.
\newblock Mechanisms of skill acquisition and the law of practice.
\newblock In Anderson, J.~R. (ed.), \emph{Cognitive Skills and Their
  Acquisition}, chapter~1, pp.\  1--51. Lawrence Erlbaum Associates, Inc.,
  Hillsdale, NJ, 1981.

\bibitem[Samuel(1959)]{Samuel59}
Samuel, A.~L.
\newblock Some studies in machine learning using the game of checkers.
\newblock \emph{IBM Journal of Research and Development}, 3\penalty0
  (3):\penalty0 211--229, 1959.

\end{thebibliography}


\begin{thebibliography}{10}

\bibitem{agarwal2020pc}
Alekh Agarwal, Mikael Henaff, Sham Kakade, and Wen Sun.
\newblock Pc-pg: Policy cover directed exploration for provable policy gradient
  learning.
\newblock {\em arXiv preprint arXiv:2007.08459}, 2020.

\bibitem{agarwal2019theory}
Alekh Agarwal, Sham~M Kakade, Jason~D Lee, and Gaurav Mahajan.
\newblock On the theory of policy gradient methods: Optimality, approximation,
  and distribution shift.
\newblock {\em arXiv preprint arXiv:1908.00261}, 2019.

\bibitem{agarwal2020boosting}
Naman Agarwal, Nataly Brukhim, Elad Hazan, and Zhou Lu.
\newblock Boosting for control of dynamical systems.
\newblock In {\em International Conference on Machine Learning}, pages 96--103.
  PMLR, 2020.

\bibitem{agarwal2022variance}
Naman Agarwal, Brian Bullins, and Karan Singh.
\newblock Variance-reduced conservative policy iteration.
\newblock {\em arXiv preprint arXiv:2212.06283}, 2022.

\bibitem{alon2020boosting}
Noga Alon, Alon Gonen, Elad Hazan, and Shay Moran.
\newblock Boosting simple learners.
\newblock {\em arXiv preprint arXiv:2001.11704}, 2020.

\bibitem{bagnell2003policy}
J~Andrew Bagnell, Sham Kakade, Andrew~Y Ng, and Jeff~G Schneider.
\newblock Policy search by dynamic programming.
\newblock In {\em Advances in Neural Information Processing Systems}, 2003.

\bibitem{beygelzimer2015optimal}
Alina Beygelzimer, Satyen Kale, and Haipeng Luo.
\newblock Optimal and adaptive algorithms for online boosting.
\newblock In {\em International Conference on Machine Learning}, pages
  2323--2331, 2015.

\bibitem{beygelzimer2011contextual}
Alina Beygelzimer, John Langford, Lihong Li, Lev Reyzin, and Robert Schapire.
\newblock Contextual bandit algorithms with supervised learning guarantees.
\newblock In {\em Proceedings of the Fourteenth International Conference on
  Artificial Intelligence and Statistics}, pages 19--26. JMLR Workshop and
  Conference Proceedings, 2011.

\bibitem{brukhim2020online}
Nataly Brukhim, Xinyi Chen, Elad Hazan, and Shay Moran.
\newblock Online agnostic boosting via regret minimization.
\newblock In {\em Advances in Neural Information Processing Systems}, 2020.

\bibitem{brukhim2021online}
Nataly Brukhim and Elad Hazan.
\newblock Online boosting with bandit feedback.
\newblock In {\em Algorithmic Learning Theory}, pages 397--420. PMLR, 2021.

\bibitem{chen2018projection}
Lin Chen, Christopher Harshaw, Hamed Hassani, and Amin Karbasi.
\newblock Projection-free online optimization with stochastic gradient: From
  convexity to submodularity.
\newblock In {\em International Conference on Machine Learning}, pages
  814--823, 2018.

\bibitem{chen2012online}
Shang-Tse Chen, Hsuan-Tien Lin, and Chi-Jen Lu.
\newblock An online boosting algorithm with theoretical justifications.
\newblock In {\em Proceedings of the 29th International Coference on
  International Conference on Machine Learning}, pages 1873--1880, 2012.

\bibitem{chen2014boosting}
Shang-Tse Chen, Hsuan-Tien Lin, and Chi-Jen Lu.
\newblock Boosting with online binary learners for the multiclass bandit
  problem.
\newblock In {\em International Conference on Machine Learning}, pages
  342--350, 2014.

\bibitem{duchi2008efficient}
John Duchi, Shai Shalev-Shwartz, Yoram Singer, and Tushar Chandra.
\newblock Efficient projections onto the l 1-ball for learning in high
  dimensions.
\newblock In {\em Proceedings of the 25th international conference on Machine
  learning}, pages 272--279, 2008.

\bibitem{frank1956algorithm}
Marguerite Frank and Philip Wolfe.
\newblock An algorithm for quadratic programming.
\newblock {\em Naval research logistics quarterly}, 3(1-2):95--110, 1956.

\bibitem{hassani2017gradient}
Hamed Hassani, Mahdi Soltanolkotabi, and Amin Karbasi.
\newblock Gradient methods for submodular maximization.
\newblock In {\em Advances in Neural Information Processing Systems}, pages
  5841--5851, 2017.

\bibitem{hazan2019introduction}
Elad Hazan.
\newblock Introduction to online convex optimization.
\newblock {\em arXiv preprint arXiv:1909.05207}, 2019.

\bibitem{hazan2019provably}
Elad Hazan, Sham Kakade, Karan Singh, and Abby Van~Soest.
\newblock Provably efficient maximum entropy exploration.
\newblock In {\em International Conference on Machine Learning}, pages
  2681--2691. PMLR, 2019.

\bibitem{hazan2021boosting}
Elad Hazan and Karan Singh.
\newblock Boosting for online convex optimization.
\newblock {\em arXiv preprint arXiv:2102.09305}, 2021.

\bibitem{jaggi2013revisiting}
Martin Jaggi.
\newblock Revisiting frank-wolfe: Projection-free sparse convex optimization.
\newblock In {\em International Conference on Machine Learning}, pages
  427--435. PMLR, 2013.

\bibitem{jin2020provably}
Chi Jin, Zhuoran Yang, Zhaoran Wang, and Michael~I Jordan.
\newblock Provably efficient reinforcement learning with linear function
  approximation.
\newblock In {\em Conference on Learning Theory}, pages 2137--2143. PMLR, 2020.

\bibitem{jung2017online}
Young~Hun Jung, Jack Goetz, and Ambuj Tewari.
\newblock Online multiclass boosting.
\newblock In {\em Advances in neural information processing systems}, pages
  919--928, 2017.

\bibitem{jung2018online}
Young~Hun Jung and Ambuj Tewari.
\newblock Online boosting algorithms for multi-label ranking.
\newblock In {\em International Conference on Artificial Intelligence and
  Statistics}, pages 279--287, 2018.

\bibitem{kakade2002approximately}
Sham Kakade and John Langford.
\newblock Approximately optimal approximate reinforcement learning.
\newblock In {\em In Proc. 19th International Conference on Machine Learning}.
  Citeseer, 2002.

\bibitem{kanade2009potential}
Varun Kanade and Adam Kalai.
\newblock Potential-based agnostic boosting.
\newblock In {\em Advances in neural information processing systems}, pages
  880--888, 2009.

\bibitem{lazaric2009hybrid}
Alessandro Lazaric and R{\'e}mi Munos.
\newblock Hybrid stochastic-adversarial on-line learning.
\newblock In {\em Conference on Learning Theory}, 2009.

\bibitem{leistner2009robustness}
Christian Leistner, Amir Saffari, Peter~M Roth, and Horst Bischof.
\newblock On robustness of on-line boosting-a competitive study.
\newblock In {\em IEEE 12th International Conference on Computer Vision
  Workshops, ICCV Workshops}, pages 1362--1369. IEEE, 2009.

\bibitem{london2022boosted}
Ben London, Levi Lu, Ted Sandler, and Thorsten Joachims.
\newblock Boosted off-policy learning.
\newblock {\em arXiv preprint arXiv:2208.01148}, 2022.

\bibitem{mokhtari2018stochastic}
Aryan Mokhtari, Hamed Hassani, and Amin Karbasi.
\newblock Stochastic conditional gradient methods: From convex minimization to
  submodular maximization.
\newblock {\em arXiv preprint arXiv:1804.09554}, 2018.

\bibitem{munos2008finite}
R{\'e}mi Munos and Csaba Szepesv{\'a}ri.
\newblock Finite-time bounds for fitted value iteration.
\newblock {\em Journal of Machine Learning Research}, 9(5), 2008.

\bibitem{pedregosa2011scikit}
Fabian Pedregosa, Ga{\"e}l Varoquaux, Alexandre Gramfort, Vincent Michel,
  Bertrand Thirion, Olivier Grisel, Mathieu Blondel, Peter Prettenhofer, Ron
  Weiss, Vincent Dubourg, et~al.
\newblock Scikit-learn: Machine learning in python.
\newblock {\em the Journal of machine Learning research}, 12:2825--2830, 2011.

\bibitem{rakhlin2011online}
Alexander Rakhlin, Karthik Sridharan, and Ambuj Tewari.
\newblock Online learning: Stochastic and constrained adversaries.
\newblock {\em arXiv preprint arXiv:1104.5070}, 2011.

\bibitem{schapire2012boosting}
Robert~E Schapire and Yoav Freund.
\newblock {\em Boosting: Foundations and Algorithms}.
\newblock MIT Press, 2012.

\bibitem{scherrer2014local}
Bruno Scherrer and Matthieu Geist.
\newblock Local policy search in a convex space and conservative policy
  iteration as boosted policy search.
\newblock In {\em Joint European Conference on Machine Learning and Knowledge
  Discovery in Databases}, pages 35--50. Springer, 2014.

\bibitem{sun2019model}
Wen Sun, Nan Jiang, Akshay Krishnamurthy, Alekh Agarwal, and John Langford.
\newblock Model-based rl in contextual decision processes: Pac bounds and
  exponential improvements over model-free approaches.
\newblock In {\em Conference on learning theory}, pages 2898--2933. PMLR, 2019.

\bibitem{Sutton1999}
Richard~S Sutton, David McAllester, Satinder Singh, and Yishay Mansour.
\newblock Policy gradient methods for reinforcement learning with function
  approximation.
\newblock In S.~Solla, T.~Leen, and K.~M\"{u}ller, editors, {\em Advances in
  Neural Information Processing Systems}, volume~12. MIT Press, 2000.

\bibitem{wang2021exponential}
Yuanhao Wang, Ruosong Wang, and Sham~M Kakade.
\newblock An exponential lower bound for linearly-realizable mdps with constant
  suboptimality gap.
\newblock {\em arXiv preprint arXiv:2103.12690}, 2021.

\bibitem{weisz2021query}
Gellert Weisz, Philip Amortila, Barnab{\'a}s Janzer, Yasin Abbasi-Yadkori, Nan
  Jiang, and Csaba Szepesv{\'a}ri.
\newblock On query-efficient planning in mdps under linear realizability of the
  optimal state-value function.
\newblock {\em arXiv preprint arXiv:2102.02049}, 2021.

\bibitem{williams1992simple}
Ronald~J Williams.
\newblock Simple statistical gradient-following algorithms for connectionist
  reinforcement learning.
\newblock {\em Machine learning}, 8(3-4):229--256, 1992.

\bibitem{xie2019stochastic}
Jiahao Xie, Zebang Shen, Chao Zhang, Hui Qian, and Boyu Wang.
\newblock Stochastic recursive gradient-based methods for projection-free
  online learning.
\newblock {\em arXiv preprint arXiv:1910.09396}, 2019.

\end{thebibliography}
\bibliographystyle{plain}

\newpage 
\onecolumn
\appendix

\section{Notation: List of Symbols}



\subsection*{Weak Learning and Boosting}

\begingroup
\renewcommand{\arraystretch}{1.5}
\setlength{\tabcolsep}{10pt} 

\begin{tabular}{ll}
$\alpha$ & Weak learning parameter\\
$T$ & Number of boosting iterations \\
$N$ & Number of internal-boosting iterations\\
$M$ & Number of internal-boosting episodes\\
$\Gamma[\cdot]$ & Policy projection \\
$\Pi$ & Policy class\\
$\mathbbl{\Pi}$ & Policy-Tree class (w.r.t $\Pi$, $\Gamma$, $N$ and $T$)\\
\end{tabular}  
\endgroup

\subsection*{Markov Decision Process}

\begingroup
\renewcommand{\arraystretch}{1.5}
\setlength{\tabcolsep}{10pt} 

\begin{tabular}{ll}
$S$ & State space\\
$A$ & Action space \\
$\Delta_A$ & Probability simplex over actions \\
$Q^\pi(s,a)$ & Q function \\
$V^\pi(s)$ & Value function \\
$d(s_0)$ & Initial state distribution\\
$d_d^\pi(s)$ & State-visitation distribution w.r.t $\pi,d$\\
$\gamma$ & Discount factor\\
$\mathcal{E}_\mu(\mathbbl{\Pi},\Pi)$ & Policy completeness\\
$\mu, \nu$ & Used for different initial state distributions \\ 
$C_\infty$ & Distribution mismatch  if $\mu = d_0$\\
$D_\infty$ & Distribution mismatch  if $\mu = \nu \neq d_0$\\
\end{tabular}  
\endgroup

\subsection*{Optimization}

\begingroup
\renewcommand{\arraystretch}{1.5}
\setlength{\tabcolsep}{10pt} 
\begin{tabular}{ll}
$\K$ & Decision set \\
$L$ & Smoothness of the objective \\
$H$ & Upper bound on the range of function value \\
$D$ & Upper bound on Euclidean diameter \\
\end{tabular}
\endgroup



\section{Appendix}

It is important that the policy that the boosting algorithm outputs can be evaluated efficiently. Towards that end, we give the following claim.

\begin{claim} \label{claim:eff_eval}
    For any $\pi \in \mathbbl{\Pi}(\Pi, N, T)$, $\pi(\cdot|s)$ for any $s\in S$ can be evaluated using $TN$ base policy evaluations and $O(T\times (NA + A\log A))$ arithmetic and logical operations.
\end{claim}
\begin{proof}
    Since $\pi \in \mathbbl{\Pi}(\Pi, N, T)$, it is composed of $TN$ base policies.
    Producing each aggregated function takes $NA$ additions and multiplications; there are $T$ of these.
    Each projection takes time equivalent to sorting $|A|$ numbers, due to a water-filling algorithm~\citep{duchi2008efficient}; these are also $T$ in number.
    The final linear transformation takes an additional $TA$ operations.
\end{proof}



     \section{RL Boosting via Weak Online Learning}

The second model of weak learning we consider requires a stronger assumption, but will give us better sample and oracle complexity bounds henceforth.

\begin{definition}[Weak Online Learner]
    \label{def:owl}
    Let $\alpha \in (0,1)$.
    Consider a class $\mathcal{L}$ of linear loss functions
    $\ell:\reals^A \to \reals$.
    A weak online learning algorithm, for every $M>0$, incrementally for each timestep computes a policy $\mathcal{W}_m\in\Pi$ and then observes the state-loss pair $(s,\ell_t)\in S\times \mathcal{L}$ such that
    \begin{align*}
        \sum_{m=1}^M \ell_m(\mathcal{W}_m(s_m)) \geq
        \alpha
        \max_{\pi^*\in\Pi}\sum_{m=1}^M\ell_m(\pi^*(s_m)) + (1-\alpha)
        \sum_{m=1}^M \ell_m(\pi_{Rand}(s_m)) - R_\mathcal{W}(M).
    \end{align*}
\end{definition}

\begin{assumption}[Weak Online Learning]
    \label{assum:owl}
    The booster has access to a weak online learning oracle (Definition \ref{def:owl}) over the policy class $\Pi$, for some $\alpha \in (0,1)$.
\end{assumption}

\begin{remark}
    A similar remark about \emph{natural} distributions applies to the online weak learner.
    In particular, it is sufficient the guarantee in \ref{def:owl} holds for arbitrary sequence of loss functions with high probability over the sampling of the state from $d^{\pi}$ for some $\pi\in\mathbbl{\Pi}$.
    Although stronger than supervised weak learning, this oracle can be interpreted as a relaxation of the online weak learning oracle considered in \citep{brukhim2020online,brukhim2021online,hazan2021boosting}.
    A similar model of hybrid adversarial-stochastic online learning was considered in \citep{rakhlin2011online,lazaric2009hybrid,beygelzimer2011contextual}.
    In particular, it is known \citep{lazaric2009hybrid} that unlike online learning, the capacity of a hypothesis class for this model is governed by its VC dimension (vs. Littlestone dimension).
\end{remark}

\begin{algorithm}[H]
    \caption{RL Boosting via Weak Online Learning}\label{alg:MAIN2}
    \begin{algorithmic}[1]
        \STATE Initialize a policy $\pi_0\in \Pi$ arbitrarily.
        \FOR{$t=1$ {\bfseries to} $T$}
        \STATE Initialize online weak learners $\W_1,\dots \W^N$.
        \FOR{$m=1$ {\bfseries to} $M$}
        \STATE Execute $\pi_{t-1}$ once with initial state distribution $\mu$ via Algorithm \ref{alg:q_sampler}, to get $(s_{t,m}, \widehat{Q}_{t,m})$.
        \STATE Choose $\Tilde{\pi}_{t,m,0}\in\Pi$ arbitrarily.
        \FOR{$n=1$ {\bfseries to} $N$}
        \STATE Set $\Tilde{\pi}_{t,m,n} = (1-\eta_{2,n})\Tilde{\pi}_{t,m,n-1} + \frac{\eta_{2,n}}{\alpha} \W^{n}$.
        \ENDFOR
        \STATE Pass to each $\W^n$ the following loss linear $f_{t,m,n}$:
        $$ f_{t,m,n} =  \frac{1}{\beta} \left(y_{t,m,n} - \Tilde{\pi}_{t,m,n}(\cdot|s_i) \right).$$
        where  $G= \frac{A}{1-\gamma}, \beta = \frac{2\gamma}{(1-\gamma)^3} $ and $f_i, \widehat{Q}_i \in \reals^{|A|}$
        \begin{align*}
        y_i = \argmin_{y\in \Delta_A}\Large\{ -\widehat{Q}_{t,m}^\top y + G\min_{z\in \Delta_A}\|z-y\|
        +\frac{\|\Tilde{\pi}_{t,m,n}(\cdot|s_{t,m})-y\|^2}{2\beta} \Large\}
        \end{align*}
        \ENDFOR
        \STATE Declare $\pi'_{t} = \frac{1}{M}\sum_{m=1}^M\Gamma\left[\Tilde{\pi}_{t,m,N}\right]$.
        \STATE Choose $\eta_{1,t} = \min\{1,\frac{2C_\infty}{t}\}$ if $\mu=d_0$ else set $\eta_{1,t} =  \sqrt{\frac{8\gamma (1-\gamma)^2}{|A|^2 T}}$.
        \STATE Update $\pi_t = (1-\eta_{1,t})\pi_{t-1} + \eta_{1,t} \pi'_t$.
        \ENDFOR
        \STATE Run each policy $\pi_t$ for $P$ rollouts to compute an empirical estimate $\widehat{V^{\pi_t}}$  of the expected return.
    \RETURN $\bar{\pi} := \pi_{t'}$ where $t'=\argmax_t \widehat{V^{\pi_t}}$.
    \end{algorithmic}
\end{algorithm}

\begin{theorem}
    \label{thm:MAIN2}
    Algorithm~\ref{alg:MAIN2} samples $TM$ episodes of length $\frac{1}{1-\gamma}\log \frac{TM}{\delta}$ with probability $1-\delta$.
    In the episodic model, Algorithm~\ref{alg:MAIN2} guarantees as long as $T=\frac{16C^2_\infty}{(1-\gamma)^3 \varepsilon}$, $N=\left(\frac{16|A|C_\infty }{(1-\gamma)^2\alpha\epsilon}\right)^2$, $M= \max\left\{ \frac{1000|A|^2C^2_\infty }{(1-\gamma)^4\varepsilon^2\alpha^2}\log^2{T}{\delta}, \frac{8|A|C_\infty R_\W(M)}{(1-\gamma)^2\alpha\varepsilon}\right\}$,$\mu=d_0$, we have with probability $1-\delta$
    $$ V^* - V^\pi \leq C_\infty  \frac{\mathcal{E}(\mathbbl{\Pi}, \Pi)}{1-\gamma} +\varepsilon$$
    In the $\nu$-reset model, Algorithm~\ref{alg:MAIN1} guarantees as long as $T=\frac{100 D^2_\infty}{(1-\gamma)^6 \varepsilon^2}$, $N=\left(\frac{20|A|D_\infty}{(1-\gamma)^3\alpha\epsilon}\right)^2$,  $M=\max\left\{\left(\frac{40|A|D_\infty}{(1-\gamma)^3\alpha\varepsilon}\log\frac{T}{\delta}\right)^2, \frac{10|A|D_\infty R_\W(M)}{(1-\gamma)^3 \alpha\varepsilon}\right\}$,$\mu=\nu$, we have with probability $1-\delta$
    $$ V^* - V^\pi \leq  D_\infty\frac{\mathcal{E}_\nu(\mathbbl{\Pi}, \Pi)}{(1-\gamma)^2} + \varepsilon$$
    If $R_\W(M)=\sqrt{M\log |\W|}$ for some measure of weak learning complexity $|\W|$, the algorithm samples $\tilde{O}\left(\frac{C_\infty^4|A|^2\log |\W|}{(1-\gamma)^{7} \alpha^2\varepsilon^3}\right)$ episodes in the episodic model, and $\tilde{O}\left(\frac{D_\infty^4|A|^2\log |\W|}{(1-\gamma)^{12} \alpha^2\varepsilon^4}\right)$ in the $\nu$-reset model.
\end{theorem}


\section{Analysis for Boosting with Weak Supervised Learning (Proof of Theorem~\ref{thm:MAIN1})}

\begin{theorem*}[Formal version of Theorem~\ref{thm:MAIN1}]
    Algorithm~\ref{alg:MAIN1} samples $TMN$ episodes of length $\frac{1}{1-\gamma}\log \frac{TMN}{\delta}$ with probability $1-\delta$.
    In the episodic model, Algorithm~\ref{alg:MAIN1} guarantees as long as $T=\frac{16C^2_\infty)}{(1-\gamma)^3 \varepsilon}$, $N=\left(\frac{16|A|C_\infty}{(1-\gamma)^2\alpha\epsilon}\right)^2$, $M=m\left(\frac{(1-\gamma)^2\alpha\varepsilon}{8C_\infty|A|},\frac{\delta}{NT}\right)$,$\mu=d_0$, we have with probability $1-\delta$
    $$ V^* - V^\pi \leq  C_\infty\frac{\mathcal{E}(\mathbbl{\Pi}, \Pi)}{1-\gamma} + \varepsilon$$
    In the $\nu$-reset model, Algorithm~\ref{alg:MAIN1} guarantees as long as $T=\frac{8D^2_\infty}{(1-\gamma)^6 \varepsilon^2}$, $N=\left(\frac{16|A|D_\infty}{(1-\gamma)^3\alpha\epsilon}\right)^2$,  $M=m\left(\frac{(1-\gamma)^3\alpha\varepsilon}{8|A|D_\infty},\frac{\delta}{2NT}\right)$,$\mu=\nu$, we have with probability $1-\delta$
    $$ V^* - V^\pi \leq  D_\infty\frac{\mathcal{E}_\nu(\mathbbl{\Pi}, \Pi)}{(1-\gamma)^2} + \varepsilon$$
    If $m(\varepsilon,\delta)=\frac{\log |\W|}{\varepsilon^2}\log \frac{1}{\delta}$ for some measure of weak learning complexity $|\W|$, the algorithm samples $\tilde{O}\left(\frac{C_\infty^6 |A|^4\log |\W|}{(1-\gamma)^{11} \alpha^4\varepsilon^5}\right)$ episodes in the episodic model, and $\tilde{O}\left(\frac{D_\infty^6|A|^4\log |\W|}{(1-\gamma)^{18} \alpha^4\varepsilon^6}\right)$ in the $\nu$-reset model.
\end{theorem*}
\begin{proof}[Proof of Theorem~\ref{thm:MAIN1}]
    The broad scheme here is to utilize an equivalence between Algorithm~\ref{alg:MAIN1} and Algorithm~\ref{fw2} on the function $V^\pi$ (or $V^\pi_\nu$ in the $\nu$-reset model), to which Theorem~\ref{fwthm} applies.

    To this end, firstly, note $V^\pi$ is $\frac{1}{1-\gamma}$-bounded. Define a norm $\|\cdot\|_{\infty,1}:\reals^{|S|\times |A|}\to \reals$ as $\|x\|_{1,\infty} = \max_{s\in S} \sum_{a\in A} |x_{s,a}|$. Further, observe that for any policy $\pi:S\to \Delta_A$, $\|\pi\|_{\infty,1}=1$. The following lemma specifies the smoothness of $V^\pi$ in this norm.
    \begin{lemma}
        \label{lem:smoothness}
        $V^\pi$ is $\frac{2\gamma}{(1-\gamma)^3}$-smooth in the $\|\cdot\|_{\infty,1}$ norm.
    \end{lemma}

    To be able to interpret Algorithm~\ref{alg:MAIN1} as an instantiation of the algorithmic template Algorithm~\ref{fw2} presents, we need to show that $\pi'_t$ (Line 3-10) serves as an approximate linear optimizer for $\nabla V^{\pi_{t-1}}$. This will imply that the iterates produced by the two algorithms coincide. 
    Indeed, Claim~\ref{thm:newstat} demonstrates that $\pi'_t$ serves a linear optimizer over gradients of the function $V^\pi$; the suboptimality specifies $\epsilon_0$.
    \begin{claim}
        \label{thm:newstat}
        Let $\beta = \sqrt{\frac{1}{\alpha N}}$, and $\eta_{2,n}=\min\{\frac{2}{n},1\}$. Then, for any $t$, $\pi'_t$ produced by Algorithm~\ref{alg:MAIN1} satisfies with probability $1-\delta$
        $$ \max_{\pi\in\Pi} (\nabla V^{\pi_{t-1}}_\mu)^\top (\pi  -  \pi'_t)  \leq \frac{2|A|}{(1-\gamma)^2\alpha} \left( \frac{2}{\sqrt{N}} + \varepsilon_W \right)  $$
    \end{claim}

    Finally, observe that it is by construction that $\pi_t\in \mathbbl{\Pi}$. Therefore, in terms of the previous section, $\K$ is the class of all policies, $\K_1=\mathbbl{\Pi}$, $\K_2=\Pi$.

    In the episodic model, we wish to invoke the second part of Theorem~\ref{fwthm}. The next lemma establishes gradient-domination properties of $V^\pi$ to support this.

    \begin{lemma}
        \label{lem:dom}
        $V^\pi$ is $\left(C_\infty,\frac{1}{1-\gamma}C_\infty\mathcal{E}(\Pi,\mathbbl{\Pi}),\mathbbl{\Pi},\Pi\right)$-gradient dominated, i.e. for any $\pi\in \mathbbl{\Pi}$:
        $$ V^* - V^\pi \leq C_\infty \left(\frac{1}{1-\gamma}\mathcal{E}(\Pi,\mathbbl{\Pi}) + \max_{\pi'\in \Pi} (\nabla V^\pi)^\top (\pi'-\pi)\right) $$
    \end{lemma}

    Deriving $\kappa,\tau$ from the above lemma along with $\epsilon_0$ from Claim~\ref{thm:newstat}, as a consequence of the second part of Theorem~\ref{fwthm}, we have with probability $1-NT\delta$
    \begin{align*}
        V^*-V^{\bar{\pi}} \leq & C_\infty\frac{\mathcal{E}(\mathbbl{\Pi}, \Pi)}{1-\gamma} +\frac{4C^2_\infty}{(1-\gamma)^3 T} + \frac{4|A|C_\infty}{(1-\gamma)^2\alpha\sqrt{N}} \\
        &+ \frac{2|A|C_\infty}{(1-\gamma)^2 \alpha}\varepsilon_W.
    \end{align*}

    Similarly, in the $\nu$-reset model, the first part of Theorem~\ref{fwthm} provides a local-optimality guarantee for $V^\pi_\nu$. Lemma~\ref{lem:domnu} provides a bound on the function-value gap (on $V^\pi$) provided such local-optimality conditions.
    \begin{lemma}
        \label{lem:domnu}
        For any $\pi\in \mathbbl{\Pi}$, we have
       \begin{align*}
      V^* - V^\pi \leq \frac{1}{1-\gamma}D_\infty \Bigg(&\frac{1}{1-\gamma}\mathcal{E}_\nu(\Pi,\mathbbl{\Pi})  + \max_{\pi'\in \Pi} (\nabla V_\nu^\pi)^\top (\pi'-\pi)\Bigg).
         \end{align*}
    \end{lemma}

    Again, using the bound on $\max_{\pi'\in \Pi} (\nabla V_\nu^{\bar{\pi}})^\top (\pi'-\bar{\pi})$ Theorem~\ref{fwthm} provides, we have that with probability $1-2NT\delta$
    \begin{align*}
        V^*-V^{\bar{\pi}} \leq \frac{D_\infty\mathcal{E}_\nu(\mathbbl{\Pi}, \Pi)}{(1-\gamma)^2} &+ \frac{2D_\infty }{(1-\gamma)^3 \sqrt{T}} \\
        &+ \frac{2|A|D_\infty}{(1-\gamma)^3 \alpha} \left(\frac{2}{\sqrt{N}}+ \varepsilon_W\right)\\
        &+ \frac{48|A|D_\infty}{(1-\gamma)^3\sqrt{P}}\log \frac{1}{\delta}
    \end{align*}
\end{proof}

     \section{Analysis for Boosting with Weak Online Learning (Proof of Theorem~\ref{thm:MAIN2})}

\begin{proof} [Proof of Theorem~\ref{thm:MAIN2}]
    Similar to the proof of Theorem~\ref{thm:MAIN1}, we establish an equivalence between Algorithm~\ref{alg:MAIN1} and Algorithm~\ref{fw2} on the function $V^\pi$ (or $V^\pi_\nu$ in the $\nu$-reset model), to which Theorem~\ref{fwthm} applies provided smoothness (see Lemma~\ref{lem:smoothness}).

    Indeed, Claim~\ref{thm:newoco} demonstrates $\pi'_t$ serves a linear optimizer over gradients of the function $V^\pi$, and provides a bound on $\epsilon_0$. As before, observe that it is by construction that $\pi_t\in \mathbbl{\Pi}$.

    \begin{claim}
        \label{thm:newoco}
        Let $\beta = \sqrt{\frac{1}{\alpha N}}$, and $\eta_{2,n}=\min\{\frac{2}{n},1\}$. Then, for any $t$, $\pi'_t$ produced by Algorithm~\ref{alg:MAIN2} satisfies with probability $1-\delta$
        \begin{align*}
         \max_{\pi\in\Pi}& (\nabla V^{\pi_{t-1}}_\mu)^\top (\pi  - \pi'_t)  \leq    \frac{2|A|}{(1-\gamma)^2\alpha} \left( \frac{2}{\sqrt{N}} + \frac{R_\W(M)}{M}  + \sqrt{\frac{16\log\delta^{-1}}{M}} \right)
        \end{align*}
    \end{claim}

    In the episodic model, one may combine the second part of Theorem~\ref{fwthm}, which provides a bound on function-value gap for gradient dominated functions, which Lemma~\ref{lem:dom} guarantees, to conclude with probability $1-T\delta$
    \begin{align*}
        V^*-V^{\bar{\pi}} \leq & \frac{C_\infty\mathcal{E}(\mathbbl{\Pi}, \Pi)}{1-\gamma} +\frac{4C^2_\infty (\mathbbl{\Pi})}{(1-\gamma)^3 T} + \frac{4|A|C_\infty}{(1-\gamma)^2\alpha\sqrt{N}} \\
        &+ \frac{2|A|C_\infty}{(1-\gamma)^2 \alpha}\frac{R_\W(M)}{M}+ \frac{8|A|C_\infty\log \delta^{-1}}{(1-\gamma)^2\alpha \sqrt{M}}.
    \end{align*}

    Similarly, in the $\nu$-reset model, Lemma~\ref{lem:domnu} provides a bound on the function-value gap provided local-optimality conditions, which the first part of Theorem~\ref{fwthm} provides for. Again, with probability $1-T\delta$
    \begin{align*}
      V^*-V^{\bar{\pi}} \leq \frac{D_\infty\mathcal{E}_\nu(\mathbbl{\Pi}, \Pi)}{(1-\gamma)^2} + \frac{2D_\infty}{(1-\gamma)^3}\Bigg(\frac{1}{\sqrt{T}}   + \frac{|A|}{\alpha} \Big(\frac{2}{\sqrt{N}} + \frac{R_\W(M)}{M} + \frac{4\log\delta^{-1}}{\sqrt{M}}\Big)  + \frac{24|A|}{\sqrt{P}}\log \frac{1}{\delta}\Bigg).
    \end{align*}
\end{proof}

     \section{Proofs of Supporting Claims}
\subsection{Guarantees on the sampling algorithm}
\begin{proof}[Proof of Lemma~\ref{lemma:q_sampler}]
    Recall $\nabla_\pi V^\pi$ denotes the gradient with respect to the  $|S|\times|A|$-sized representation of the policy $\pi$ -- the functional gradient.
    Then, using the policy gradient theorem \citep{williams1992simple, Sutton1999}, it is given by,
    \begin{equation}\label{eq:grad}
        \frac{\partial V_\mu^\pi}{\partial \pi (a|s)} = \frac{1}{1-\gamma} d^{\pi}_\mu(s) Q^{\pi}(s,a).
    \end{equation}

    The following sources of randomness are at play in the sampling algorithm (Algorithm~\ref{alg:q_sampler}): the distribution $d^{\pi}$ (which encompasses the discount-factor-based random termination, the transition probability, and the stochasticity of $\pi$), and the uniform sampling over $A$. For a fixed $s, \pi$, denote by $\mathcal{Q}^\pi_s$ as the distribution over $\widehat{Q^\pi}(s,\cdot) \in \mathbb{R}^A$, induced by all the aforementioned randomness sources. To conclude the claim, observe that by construction
    \begin{equation}
        \label{eq:unbias}
        \E_{\mathcal{Q}^\pi(s,\cdot)}[\widehat{Q^\pi}(s,\cdot) | \pi, s] = Q^\pi(s,\cdot).
    \end{equation}
\end{proof}
\subsection{Non-convex Frank-Wolfe method (Theorem~\ref{fwthm})}
\begin{proof}[Proof of Theorem~\ref{fwthm}]
    \textbf{Non-convex general case.} Note that for any timestep $t$, it holds due to smoothness that
    \begin{align} 
        f(x_{t})  &=  f( x_{t-1}  + \eta  ( {z_t} - x_{t-1}) )  \\
        &  \geq  f(x_{t-1} )   + \eta   \nabla_{t-1}^\top( z_{t} - x_{t-1})  - \eta^{2} \frac{L}{2} D^2. \label{eq:ee1}
    \end{align}
    Let $t'=\argmin_t f(x_{t}) - f(x_{t-1})$. 
    Note that by telescoping over function-value differences across successive iterates, we get
    \begin{align*}
    f(x_{t'}) - f(x_{t'-1}) \le \frac{1}{T}\Big( f(x_T) - f(x_0)\Big) \le \frac{2H}{T}.
    \end{align*}
Combining with \eqref{eq:ee1}, and plugging in $\eta$, we get
    \begin{align*}
        \nabla_{t'-1}^\top( z_{t'} - x_{t'-1}) &\leq \eta  LD^2/2 + \frac{2H}{T \eta } \\
        &\leq  \sqrt{\frac{2LD^2H}{T}}.
    \end{align*}

    To conclude the claim for the non-convex general case, observe that since $z_{t'}=\mathcal{O}(\nabla_{t'-1})$, it follows by the oracle definition that
    $$  \max_{u\in \K_2}\nabla_{t'-1}^\top u \leq \nabla_{t'-1}^\top z_{t'} + \epsilon_0. $$

    \paragraph{Gradient-dominated case.}
    Let $x^* = \argmax_{x\in\K} f(x)$ and let $h_t = f(x^*)-f(x_t)$.
    \begin{align*}
        h_{t}  & \leq h_{t-1} - \eta_t \nabla_{t-1}^\top ( z_t - x_{t-1}) + \eta_t^2 \frac{L}{2} D^2 \\
        &\mbox{(by smoothness)} \\
        & \leq h_{t-1} -\eta_t \max_{y\in \K_2}\eta_t \nabla_{t-1}^\top ( y - x_{t-1}) + \eta_t^2 \frac{L}{2} D^2 +\eta_t\epsilon_0 \\
        & \mbox{(by oracle guarantee)}  \\
        & \leq  h_{t-1} - \frac{\eta_t}{\kappa} (f(x^*)-f(x_{t-1})) + \eta_t^2 \frac{L}{2} D^2 +\eta_t\left(\epsilon_0+\frac{\tau}{\kappa}\right)\\
        & \mbox{(by gradient domination)} \\
        & = \left(1 - \frac{\eta_t}{\kappa}\right) h_{t-1} + \eta_t^2 \frac{L}{2} D^2 +\eta_t\left(\epsilon_0+\frac{\tau}{\kappa}\right).
    \end{align*}
    The theorem then follows from the following claim.
    \begin{claim}\label{claim:fwlr}
        Let $C\geq 1$. Let $g_t$ be a $H$-bounded positive sequence such that
    $$g_{t}\leq \left(1-\frac{\sigma_t}{C}\right)g_{t-1} + \sigma_t^2 D + \sigma_t E .$$
        Then choosing $\sigma_t = \min\{1, \frac{2C}{t}\}$ implies $g_t \leq \frac{2C^2\max\{2D, H\}}{t} + CE$.
    \end{claim}
\end{proof}

\subsection{Smoothness of value function (Lemma~\ref{lem:smoothness})}
\begin{proof}[Proof of Lemma~\ref{lem:smoothness}]
    Consider any two policies $\pi, \pi'$. Using the Performance Difference Lemma (Lemma 3.2 in \citep{agarwal2019theory}, e.g.) and Equation~\ref{eq:grad}, we have
    \begin{align*}
        | V^{\pi'} -& V^{\pi} -  \nabla V^\pi (\pi'-\pi) | \\
        &= \frac{1}{1-\gamma} \Big| \E_{s\sim d^{\pi'}} \left[Q^\pi(\cdot|s)^\top (\pi'(\cdot|s)-\pi(\cdot|s)\right] \\
        & \qquad - \E_{s\sim d^{\pi}} \left[Q^\pi(\cdot|s)^\top (\pi'(\cdot|s)-\pi(\cdot|s)\right] \Big| \\
        &\leq  \frac{1}{(1-\gamma)^2} \|d^{\pi'}-d^\pi\|_1 \|\pi'-\pi\|_{\infty,1}.
    \end{align*}
    The last inequality uses the fact that $\max_{s,a} Q^\pi(s,a) \leq \frac{1}{1-\gamma}$. It suffices to show $\|d^{\pi'}-d^\pi\|_1 \leq \frac{\gamma}{1-\gamma}\|\pi'-\pi\|_{\infty,1}$. To establish this, consider the Markov operator $P^\pi(s'|s) = \sum_{a\in A} P(s'|s,a)\pi(a|s)$ induced by a policy $\pi$ on MDP $M$. For any distribution $d$ supported on $S$, we have
    \begin{align*}
        \| (P^{\pi'}&-P^{\pi}) d\|_1 \\
        &= \sum_{s'}\left|\sum_{s,a} P(s'|s,a) d(s) (\pi'(a|s) - \pi(a|s)\right|\\
        &\leq \sum_{s'} P(s'|s,a) \|d\|_1 \|\pi'-\pi\|_{\infty,1} \\
        &\leq \|\pi'-\pi\|_{\infty,1}.
    \end{align*}
    Using sub-additivity of the $l_1$ norm and applying the above observation $t$ times, we have for any $t$
    $$ \| ((P^{\pi'})^t-(P^{\pi})^t) d\|_1  \leq t \|\pi'-\pi\|_{\infty,1}. $$
    Finally, observe that
    \begin{align*}
        \|d^{\pi'}-d^\pi\|_1 &\leq (1-\gamma)\sum_{t=0}^\infty \gamma^t  \| ((P^{\pi'})^t-(P^{\pi})^t) d_0\|_1 \\
        &\leq \|\pi'-\pi\|_{\infty,1} (1-\gamma) \sum_{t=0}^\infty t \gamma^t \\
        &= \frac{\gamma}{1-\gamma} \|\pi'-\pi\|_{\infty,1}.
    \end{align*}
\end{proof}


\subsection{Gradient domination (Lemma~\ref{lem:dom} and Lemma~\ref{lem:domnu})}
\begin{proof}[Proof of Lemma~\ref{lem:dom}]
    Invoking Lemma 4.1 from \citep{agarwal2019theory} with $\mu=d_0$, we have
    \begin{align*}
        V^* - V^\pi &\leq \left\|\frac{d^{\pi^*}}{d^\pi}\right\|_\infty  \max_{\pi_0} (\nabla V^\pi)^\top (\pi_0-\pi)  \\
        &\leq C_\infty  ( \max_{\pi_0} (\nabla V^\pi)^\top \pi_0 - \max_{\pi'\in \Pi} (\nabla V^\pi)^\top \pi' \\
        & \qquad +   \max_{\pi'\in \Pi} (\nabla V^\pi)^\top (\pi'-\pi) ).
    \end{align*}
    Finally, with the aid of Equation~\ref{eq:grad}, observe that
    \begin{align*}
        \max_{\pi_0} & (\nabla V^\pi)^\top \pi_0 - \max_{\pi'\in \Pi} (\nabla V^\pi)^\top \pi' \\
        &= \min_{\pi'\in \Pi} \frac{1}{1-\gamma} \E_{s\sim d^\pi} \left[ \max_a Q^\pi(s,a) - Q^\pi(\cdot|s)^\top \pi' \right]\\
        &\leq \frac{1}{1-\gamma} \mathcal{E}(\Pi,\mathbbl{\Pi}).
    \end{align*}
\end{proof}

\begin{proof}[Proof of Lemma~\ref{lem:domnu}]
    Invoking Lemma 4.1 from \citep{agarwal2019theory} with $\mu=\nu$, we have
    \begin{align*}
        V^* &- V^\pi \\
        &\leq \frac{1}{1-\gamma}\left\|\frac{d^{\pi^*}}{\nu}\right\|_\infty  \max_{\pi_0} (\nabla V_\nu^\pi)^\top (\pi_0-\pi)  \\
        &\leq \frac{1}{1-\gamma} D_\infty ( \max_{\pi_0} (\nabla V_\nu^\pi)^\top \pi_0 - \max_{\pi'\in \Pi} (\nabla V_\nu^\pi)^\top \pi' \\
        & \qquad +   \max_{\pi'\in \Pi} (\nabla V_\nu^\pi)^\top (\pi'-\pi) ).
    \end{align*}
    Again, with the aid of Equation~\ref{eq:grad}, observe that
    \begin{align*}
        \max_{\pi_0} &(\nabla V_\nu^\pi)^\top \pi_0 - \max_{\pi'\in \Pi} (\nabla V_\nu^\pi)^\top \pi' \\
        &= \min_{\pi'\in \Pi} \frac{1}{1-\gamma} \E_{s\sim d_\nu^\pi} \left[ \max_a Q^\pi(s,a) - Q^\pi(\cdot|s)^\top \pi' \right]\\
        &\leq \frac{1}{1-\gamma} \mathcal{E}_\nu(\Pi,\mathbbl{\Pi}).
    \end{align*}
\end{proof}

\subsection{Supervised linear optimization guarantees}
\begin{proof}[Proof of Claim~\ref{claim:oldstat}]
    The internal boosting subroutine of Algorithm \ref{alg:MAIN1}, that is
    presented in Algorithm \ref{alg:MAIN2}, is an instantiation of Algorithm 3 from \citep{hazan2021boosting}, specializing the decision set to be $\Delta_A$. To note the equivalence, note that in \citep{hazan2021boosting} the algorithm is stated assuming that the center-of-mass of the decision set is at the origin (after a coordinate transform); correspondingly, the update rule in Algorithm~\ref{alg:MAIN1} can be written as
    $$ (\Tilde{\pi}_{n}-\pi) = (1-\eta_{2,n})(\Tilde{\pi}_{n-1}-\pi) + \frac{\eta_{2,n}}{\alpha} (\mathcal{A}_{t,n} - \pi) . $$
For any state $s$, $\pi(\cdot|s)=\frac{1}{A}\mathbf{1}_{|A|}$ corresponds to the center-of-mass of $\Delta_A$.
    Finally, note that maximizing $f^\top x$ over $x\in \K$ is equivalent to minimizing $(-f)^\top x$ over the same domain.
    Therefore, we can apply previous result on boosting for statistical learning from \citep{hazan2021boosting} (Theorem 13).
    Note that $\widehat{Q^\pi}(s, \cdot)$ produced by Algorithm~\ref{alg:q_sampler} satisfies $\|\widehat{Q^\pi}(s, \cdot)\| = \frac{|A|}{1-\gamma}$.
    Let $\mathcal{D}_t$ be the distribution induced by the trajectory sampler in round $t$. This yields the bound in the claim.
\end{proof}

\begin{proof}[Proof of Claim~\ref{thm:newstat}]
    
    Lemma \ref{lemma:q_sampler} allows us to restate the guarantees from Claim \ref{claim:oldstat}  in terms of linear optimization over functional gradients.
    The conclusion thus follows immediately by combining Lemma~\ref{lemma:q_sampler} and Theorem~\ref{claim:oldstat}.
\end{proof}

\subsection{Online linear optimization guarantees (Claim~\ref{thm:newoco})}
\begin{proof}[Proof of Claim~\ref{thm:newoco}]
    In a similar vein to the proof of Claim~\ref{thm:newstat}, here we state the
    a result on boosting for online convex optimization (OCO) from \citep{hazan2021boosting} (Theorem 6), the counterpart of the result used above for the online weak learning case.

    \begin{theorem}
        \label{thm:oldoco}
        Let $\beta = \sqrt{\frac{1}{\alpha N}}$, and $\eta_{2,n}=\min\{\frac{2}{n},1\}$. Then, for any $t$, $\Gamma[\Tilde{\pi}_{t,m,N}]$ produced by Algorithm~\ref{alg:MAIN2} satisfies
   \begin{align*}
   \max_{\pi\in\Pi}\sum_{m=1}^M & \left[ \hat{Q}_{t,m}^\top \pi(s_{t,m}) \right]  - \sum_{m=1}^M \left[ \hat{Q}_{t,m}^\top \Gamma[\Tilde{\pi}_{m,N}](s_{t,m}) \right] \leq \frac{2|A|}{(1-\gamma)\alpha} \left( \frac{2M}{\sqrt{N}} + R_\W(M) \right).
        \end{align*} 
    \end{theorem}

Next we invoke online-to-batch conversions. Note that in Algorithm~\ref{alg:MAIN2}, $(s_{t,m}, \hat{Q}_{t,m})$ for any fixed $t$ is sampled i.i.d. from the same distribution.
    Therefore, we can apply online-to-batch results, i.e. Theorem 9.5 in \citep{hazan2019introduction}, on Theorem~\ref{thm:oldoco} to get
    \begin{align*}
         \max_{\pi\in\Pi}& \mathbb{E}_{(s,Q)\sim \mathcal{D}_t} \left[ Q^\top \pi(s) \right]  - \mathbb{E}_{(s,Q)\sim \mathcal{D}_t} \left[ Q^\top \pi'_t(s) \right]  \leq \frac{2|A|}{(1-\gamma)\alpha} \left( \frac{2}{\sqrt{N}} + \frac{R_\W(M)}{M}  + \sqrt{\frac{16\log\delta^{-1}}{M}} \right).
    \end{align*}
    We finally invoke Lemma~\ref{lemma:q_sampler}.
\end{proof}

\subsection{Remaining proofs (Claim~\ref{claim:fwlr})}
\begin{proof}[Proof of Claim~\ref{claim:fwlr}]
    Let $T^*=\argmax_{t} \{t: t\leq 2C\}$. For any $t\leq T^*$, we have $\sigma_t=1$ and $g_t\leq H \leq \frac{2C^2H}{t}$. For $t\geq T^*$, we proceed by induction. The base case ($t=T^*$) is true by the previous display. Now, assume $g_{t-1} \leq \frac{2C^2\max\{2D, H\}}{t-1} + CE$ for some $t> T^*$.
    \begin{align*}
        g_{t} &\leq \left(1-\frac{2}{t}\right)\left( \frac{2C^2\max\{2D, H\}}{t-1} + CE\right) \\
        & \qquad + \frac{4C^2D}{t^2}  +  \frac{2CE}{t} \\
        &\leq CE + 2C^2\max\{2D,H\}\left( \frac{1}{t-1}\left(1-\frac{2}{t}\right) + \frac{1}{t^2} \right)\\
        &= CE + 2C^2\max\{2D,H\} \frac{t^2-2t+t-1}{t^2(t-1)} \\
        &\leq CE + 2C^2\max\{2D,H\} \frac{t(t-1)}{t^2(t-1)}.
    \end{align*}
\end{proof}

\end{document}